\documentclass[twoside,11pt]{article}
\usepackage{jair, theapa, rawfonts}
\usepackage{amsmath,amsthm,amssymb}
\usepackage{extarrows}
\usepackage{graphicx}
\usepackage{subfigure}
\usepackage{color}
\usepackage{algorithm,algorithmic}
\usepackage[section]{placeins}
\DeclareMathOperator*{\argmax}{arg\,max}

\newtheorem{theorem}{Theorem}

\ShortHeadings{Pretrain Soft Q-Learning with Imperfect Demonstrations}
{Zhang, Li, Ma, \& Luo}
\firstpageno{1}

\begin{document}

\title{Pretrain Soft Q-Learning with Imperfect Demonstrations}

\author{\name Xiaoqin Zhang \thanks{These authors contributed equally to this work.} \email xiaoqin-15@mails.tsinghua.edu.cn\\
       \name Yunfei Li \footnotemark[1] \email l-yf16@mails.tsinghua.edu.cn\\
       \name Huimin Ma \thanks{Corresponding author.} \email mhmpub@tsinghua.edu.cn \\
       \addr Department of Electronic Engineering \\Tsinghua University \\Beijing, China
       \AND
       \name Xiong Luo \email xluo@ustb.edu.cn \\
       \addr School of Computer and Communication Engineering \& Institute of Artificial Intelligence \\University of Science \& Technology Beijing \\Beijing, China
       }


\maketitle

\begin{abstract}
Pretraining reinforcement learning methods with demonstrations has been an important concept in the study of reinforcement learning since a large amount of computing power is spent on online simulations with existing reinforcement learning algorithms. 
Pretraining reinforcement learning remains a significant challenge in exploiting expert demonstrations whilst keeping exploration potentials, especially for value based methods. 
In this paper, we propose a pretraining method for soft Q-learning. Our work is inspired by pretraining methods for actor-critic algorithms since soft Q-learning is a value based algorithm that is equivalent to policy gradient. The proposed method is based on $\gamma$-discounted biased policy evaluation with entropy regularization, which is also the updating target of soft Q-learning. Our method is evaluated on various tasks from Atari 2600. Experiments show that our method effectively learns from imperfect demonstrations, and outperforms other state-of-the-art methods that learn from expert demonstrations.
\end{abstract}

\section{Introduction}
\label{Introduction}

  Reinforcement Learning (RL) solves Markov Decision Process (MDP) problems based on exploration and evaluation. 
  The development of deep neural networks further enables RL algorithms to master complex sequential control problems \cite{mnih2015human,silver2016mastering}.
  Mainstream model free deep RL methods include policy based \cite{sutton2000policy,lillicrap2015continuous,mnih2016asynchronous} and value based methods \cite{mnih2015human,hessel2017rainbow}, and both 
  have made significant progress on various tasks.   
  
  However, model free RL methods typically consume dramatic amount of interaction with the environment for trial-and-error \cite{sutton2018reinforcement} before a meaningful policy is grasped. The low sample efficiency makes them computational expensive to train.
  Take AlphaZero \cite{silver2018general} as an example: up to 5000 first-generation TPUs are used for online simulations. 
  To address sample efficiency challenge, it is useful to effectively learn from datasets obtained from experts, hence decrease the number of simulations, and save training time. 
  
  Some existing work focuses on making use of expert trajectories. DQfD and POfD \cite{hester2017deep,kang2018policy} are based on inverse reinforcement learning (IRL), Lakshminarayanan et al. and Nachum et al. \cite{lakshminarayanan2016reinforcement,nachum2017bridging} train value based RL algorithms using behavior cloning (BC) losses.  
  
  These methods rely on the perfection of demonstrations to achieve good performance. However, expert demonstrations are not guaranteed to be perfect for a given task because optimal policy can be hard to obtain in real world. 
  
  Furthermore, expert demonstrations may be different from online data in that reward signals are likely to be missing because it is difficult to define reward functions in real world applications such as driving.
  In this paper, reward signals are excluded from expert demonstrations, which makes this work distinctive from those utilizing demonstrations with reward.
  

  Zhang and Ma \cite{zhang2018pretraining} successfully pretrain one of the policy based RL algorithms, i.e. actor-critic, utilizing imperfect expert demonstrations without reward signals, 
  but such method is not directly compatible with value based RL methods. However, soft Q-learning as a special case of value based RL enjoys a theoretically good property of being equivalent to policy gradient methods \cite{schulman2017equivalence}. Inspired by this equivalence, we could extend policy based pretrain method to value based soft Q-learning in this work.
  
  In this paper, we give the first method that pretrains one of the value based RL methods, soft Q-learning, with imperfect and reward-missing demonstrations. 
  The main contributions of this paper are:
  
  \begin{itemize}
  	
  	\item decoupling policy function and value function in the policy evaluation of soft Q-learning to update policy and value functions respectively with demonstrations.
  	
  	\item expediting learning with expert demonstrations whilst placing no restrictions on asymptotic performance despite the imperfection of demonstrations. 
  \end{itemize}
  
  Our proposed pretrain method is evaluated on various environments from Atari 2600 games. Experiments show that our method effectively prepares RL for further increase when pretrain finishies and improves sample efficiency by learning from imperfect demonstrations, and outperforms both IRL and BC methods.

\section{Related Work}
Some recent work focuses on making use of expert trajectories. For policy based methods, Kang et al. \cite{kang2018policy} use GANs to imitate experts. Their work is based on IRL and aims at learning from imperfect demonstrations in environments where reward signals are sparse and rare. The first published version of AlphaGo \cite{silver2016mastering}, Nair et al. \cite{nair2018overcoming} and Rajeswaran et al. \cite{rajeswaran2017learning} apply BC methods to learn from expert demonstrations, and train policy functions as classification or regression tasks. These methods focus on mimicking demonstrations and rely on the perfection of demonstrations to achieve good performance.

DDPGfD \cite{vecerik2017leveraging} adds expert demonstrations to replay buffers of online trajectories and learns with modified DDPG losses, but demonstrations used in this work are trajectories with reward signals, which is a different setting from our work. 

Zhang and Ma \cite{zhang2018pretraining} pretrain the actor-critic networks using policy based gradients. It succeeds in warming up actor-critic RL algorithms with imperfect demonstrations, but it is incompatible for value based methods.

For value based methods,
DQfD \cite{hester2017deep} learns from expert demonstrations via IRL, with the assumption that experts are the global optimum. DQfD is obtained from a large margin IRL constraint \cite{piot2014boosted} and learns from demonstrations by assuming the experts are optimal.
Lakshminarayanan et al. \cite{lakshminarayanan2016reinforcement} train DQN with expert demonstrations using BC, by applying cross-entropy loss to Q networks, to update implicit policies of Q-learning, therefore BC is one of the heuristic methods to introduce expert demonstrations. 

Brys et al. \cite{brys2015reinforcement} propose a method to learn from demonstrations using reward shaping. They propose a potential function that encourages policies to learn from demonstrations and not to disturb the optimal policy of the system. However, the potential function of the method is defined to search the whole demonstration dataset each time it is called. Consequently, the method cannot scale to tasks with high-dimensional state spaces and large demonstration datasets.

Nachum et al. \cite{nachum2017bridging} also introduce expert trajectories to soft Q-learning process using Behavior Cloning losses, which have better results than original soft Q-learning. Since soft Q-learning is equivalent to policy gradient methods, an explicit policy function is provided with Q functions, and Behavior Cloning methods can train the explicit policy function with expert demonstrations.

\section{Preliminaries}
\subsection{Soft Q-Learning}
Soft Q-learning is a kind of Q-learning that augments the standard value function with a discounted entropy \cite{haarnoja17reinforcement,nachum2017bridging} or KL divergence \cite{schulman2017equivalence} regularizer to encourage exploration, then the training loss based on Bellman equation is differentiable. For convenience, we only use entropy regularizer in this paper, but the results of our analysis can be extended to value functions with KL divergence regularizer. 

In this paper, we define the trajectory $\tau:=\{s_{1:\infty},a_{0:\infty}\}$, and $\tau_t :=\{s_{t+1:\infty},a_{t:\infty}\}$, entropy function $h^\pi_s$ $=H(\pi(\cdot|s))$ $=-\sum_a\pi(a|s)\log \pi(a|s)$, and $r^\pi_t = r_t + \epsilon h^\pi_{s_t}$.

With entropy regularizer, the Q function is defined as

\begin{equation}\label{biasq}
Q^{\pi}(s_t,a_t) = \mathbb{E}_{\tau_{t+1}, s_{t+1}\sim \pi}\left[r_0 + \sum_{l=1}^{\infty}{\gamma^lr^{\pi}_{t+l}}\right],
\end{equation}

the value function is $V^\pi(s) = \mathbb{E}_{a\sim\pi}Q^\pi(s,a)$, and the advantage function is $A^\pi(s_t,a_t) = Q^\pi(s_t,a_t)-V^\pi(s_t)$. The Bellman equation of an MDP with entropy regularization is

\begin{equation*}
Q^*(s,a) = \mathbb{E}_{s',r|s,a}\Big[r + \left.\gamma\max_{\pi}\left({\mathbb{E}_{a'\sim \pi}\left[Q^\pi(s',a')\right]}+\epsilon h^\pi_{s'}\right)\right].
\end{equation*}

Not that $\mathbb{E}_{a'\sim \pi}\left[Q^\pi(s',a')\right] +\epsilon h^\pi_{s'}$ can be regarded as a scaled KL divergence between $\pi$ and $\pi^*$, where $\pi^*$ is the optimal policy for the MDP with entropy regularization. A differentiable version of Bellman equation can then be provided:

\begin{equation*}
Q^*(s,a) = \mathbb{E}_{s',r|s,a}\left[r+\gamma\epsilon\log\sum_{a'} \exp \left(Q^*(s',a')/\epsilon\right) \right].
\end{equation*}

In the learning process of soft Q-learning, Q functions are estimated with neural networks. Suppose the parameters of the estimated Q function are $\theta$, and define
\begin{equation*}
y_t = r_t+\gamma\epsilon\log \sum_{a}\exp\left(Q_\theta(s_{t+1},a)/\epsilon\right),
\end{equation*}
the loss of soft Q-learning is

\begin{equation*}
L(Q_\theta) = \mathbb{E}_{s_t,a_t,t}\left[\frac{1}{2}\left(Q_\theta(s_t,a_t)-y_t\right)^2\right],
\end{equation*}
which is differentiable as well.

Schulman et. al. \cite{schulman2017equivalence} prove that soft Q-learning is actually equivalent to policy gradient methods with entropy regularization, which we present in Section \ref{gradient}. The policy function that soft Q-learning is updating is the softmax of scaled Q values,
\begin{equation}\label{pitheta}
\pi_\theta(a|s) = \frac{\exp{Q_\theta(s,a)}/\epsilon}{\sum_{a'}\exp{Q_\theta(s,a')}/\epsilon}.
\end{equation}

As Schulman et al. proved,

\begin{equation}\label{equv}
\begin{split}
g^{soft-Q}=&\nabla_\theta\mathbb{E}_{s_t,a_t,t}\left.\left[\frac{1}{2}\left(Q_\theta(s_t,a_t)-y_t \right)^2\right] \right|_{\pi=\pi_\theta}\\
=&\mathbb{E}_{s_t,a_t,t}\left.\left[-\epsilon g_{\theta, t}+\nabla_\theta\frac{1}{2}\delta_V^2\right]\right|_{\pi=\pi_\theta},
\end{split}
\end{equation}
where
\begin{gather*}
g_{\theta, t} = Q_\theta(s_t,a_t)\nabla_\theta\log\pi_\theta(a_t|s_t)+\epsilon\nabla_\theta h^{\pi_\theta}_{s_t},\\
\delta_V = V_\theta(s_t)-V_{target}.
\end{gather*}

The result shows that the gradient of soft Q-learning loss has the form of a combination of a policy gradient with entropy regularization and a value function update.

\subsection{Undiscounted Policy Evaluation and Policy Gradient with Entropy Regularization}
\label{gradient}

Policy gradient methods are based on policy evaluations, and the policy gradient with entropy in Equation \ref{equv} is based on policy evaluation with entropy. 
Note that the derivation in the following sections is based on undiscounted policy evaluation ($\gamma=1$) for simplicity; the result can be modified to a $\gamma$-discounted version (which we refer to as biased policy estimation) by inserting $\gamma$ factors.

Here we give a brief description of undiscounted policy evaluation. 
Schulman et al. give a general form of policy gradient with undiscounted policy evaluation in \cite{schulman2015high}, and give a policy evaluation with entropy (or KL divergence), then we have undiscounted policy evaluation with entropy

\begin{equation}\label{unbiaseval}
\eta(\pi) = \mathbb{E}_{\tau\sim \pi}\left[\sum_{t=0}^{T\rightarrow\infty}{r^\pi_t}\right],
\end{equation}
the undiscounted Q function with entropy

\begin{equation}\label{unbiasq}
Q^\pi(s_t,a_t) = \mathbb{E}_{\tau_{t+1}, s_{t+1}\sim \pi}\left[r_t+\sum_{l=1}^{T\rightarrow\infty}{r^\pi_{t+l}}\right],
\end{equation}
and the undiscounted policy gradient with entropy regularization $g=\nabla_{\theta}\eta(\pi_\theta)$ is

\begin{equation}\label{unbiasgrad}
\begin{split}
g = \mathbb{E}_{\tau\sim\pi_\theta}\left[\sum_{t=0}^{\infty}Q^{\pi_\theta}(s_t, a_t)\nabla_\theta\log \pi_\theta(a_t|s_t)+\epsilon \nabla_\theta h^{\pi_\theta}_{t}\right].
\end{split}
\end{equation}

In soft Q-learning, the biased $\gamma$-discounted estimation of Q function is defined in Equation (\ref{biasq}). With this estimation for the Q function in the policy gradient in Equation (\ref{unbiasgrad}), the policy gradient is the same as the one in Equation (\ref{equv}).

\section{Learning from Demonstration with Imperfect Data}
Zhang and Ma \cite{zhang2018pretraining} proposed a method to learn with imperfect demonstrations in policy gradient based actor-critic algorithms. 
As the training gradient of soft Q-learning $g^{soft-Q}$ can be equivalently represented as a combination of policy gradient and value function updating (see Equation \ref{equv}), it is natural to decouple soft Q into policy and value learning similar to actor-critic, thus extending Zhang and Ma's pretrain method to soft Q-learning.

\subsection{Learning from Imperfect Expert Demonstrations without Reward Signals}
Based on the undiscounted policy evaluation (\ref{unbiaseval}), we can calculate the evaluation of current policy with expert demonstrations, using the following theorem:

\begin{theorem}\label{theo}
	In the settings of undiscounted policy evaluation with entropy, we have policy evaluation with demonstrations:
	\begin{equation*}
	\eta(\pi) = \eta(\pi^*)-\mathbb{E}_{\tau\sim\pi^*}\left[\sum^{\infty}_{t=0}A^\pi(s_t^*,a_t^*) + \epsilon h^{\pi^*}_{s_t^*}\right].
	\end{equation*}
	\begin{equation*}
	\mathbb{E}_{\tau\sim\pi_\theta}\left[\sum_{t=0}^\infty r_t\right]
	=\mathbb{E}_{\tau\sim\pi^*}\left[\sum_{t=0}^\infty r_t\right] - \left[\mathbb{E}_{\tau\sim\pi^*}\left[\sum_{t=0}^{\infty}A^{\pi}(s_t^*,a_t^*)\right] + \mathbb{E}_{\tau\sim\pi_\theta} \left[ \sum_{t=0}^{\infty} \epsilon h^{\pi}_{s_t}\right]\right].
	\end{equation*}
	
\end{theorem}

Proof of Theorem \ref{theo} is presented in Appendix \ref{proof:th1}.

Note that $\eta(\pi^*)$ and $\mathbb{E}_{\tau\sim\pi^*}\left[\sum_{t=0}^\infty r_t\right]$ are constant when back-propagating into $\pi$ and Q. Hence Theorem \ref{theo} can be used to evaluate policy,
calculate policy gradient and update value function.

\subsection{Decoupling Soft Q-learning into Policy and Value Learning}


Here we propose a method to decouple policy function and Q function in policy evaluation defined in Theorem \ref{theo}. The rationale for the decoupling is given at first, and then a detailed description of how it is done is presented. 

Since soft Q-learning is a value based RL method, the policy function is implicit in Q function (Equation \ref{pitheta}) and policy updating is actually accomplished by updating Q function. However, directly propagating $\eta(\pi)$ into Q function can be problematic because this form of policy is optimal only if Q is accurate, but the assumption is not true especially in the initial training phase. In fact, directly propagating $\eta(\pi)$ into Q value is equivalent to BC method, which is proved in section \ref{discussion}.

The aim of our decoupling method is to update policy and Q function separately.
By stopping the gradient of policy or Q function, we can update the other by calculating its partial derivative to the corresponding loss function. 

Suppose Q function of soft-Q learning is neural network with parameters $\theta$. With definitions in Section \ref{gradient}, we have $A_\theta(s,a)=Q_\theta(s,a)-V_\theta(s)$, and we can obtain from definitions that
\begin{equation*}
V_\theta(s) = \sum_{a'}\pi_\theta(a'|s)Q_\theta(s,a')+\epsilon h^{\pi_\theta}_s.
\end{equation*}
Then we can give a formulation of advantage function that is only denoted by policy function and Q function:

\begin{equation}\label{adv_pq}
A_\theta(s,a) = Q_\theta(s,a)-\sum_{a'}\pi_\theta(a'|s)Q_\theta(s,a')-\epsilon h^{\pi_\theta}_s.
\end{equation}

Then by regarding one as a constant (which we denote by $\overline{x}$), and update the other with its gradient. In this way, we decouple the policy function and Q function in the training process based on expert demonstrations. 

Similar to policy gradient, loss of $\pi$ is defined as 
\begin{equation*}
L_\pi=\mathbb{E}_{\tau\sim\pi^*}\left[\sum_{t=0}^{\infty}A_\theta(s_t^*,a_t^*)\right] + \mathbb{E}_{\tau\sim\pi_\theta} \left[ \sum_{t=0}^{\infty} \epsilon h^{\pi_\theta}_{s_t}\right].
\end{equation*}  
Because soft Q-learning is equivalent to actor-critic algorithms, by stopping the gradient of the Q function, we can calculate the policy gradient based on expert demonstrations:

\begin{equation*}
\begin{split}
g_\pi &= \left.-\nabla_{\theta}L_\pi\right|_{stop-Q}\\
&=\mathbb{E}_{\tau^*\sim\pi^*,\tau\sim\pi, a'\sim \pi_{\theta}(a'|s_t^*)}\sum_{t=0}^\infty \left[\overline{Q_{\theta}(s_t^*,a')}\nabla_{\theta}\log \pi_{\theta}(a'|s_t^*)+\epsilon \nabla_{\theta}\left( h_{s_t^*}^{\pi_{\theta}} -h_{s_t}^{\pi_{\theta}}\right)\right].
\end{split}
\end{equation*}
With this policy gradient, we can update the policy function defined in Equation (\ref{pitheta}).

In Equation (\ref{equv}), soft Q-learning also updates value function besides policy function. This means that the policy gradient method that soft Q-learning is equivalent to has a value function that is not accurate enough, and needs to learn from simulated data.

Hence it is necessary to update Q function with expert demonstrations to increase the accuracy of the value function. Similar to \cite{zhang2018pretraining}, we constrain expert policies $\pi^*$ to \emph{perform better} than $\pi_\theta(\cdot|\cdot)$ defined in Equation (\ref{pitheta}). Although we already have a policy evaluation function $\eta(\pi)$, it is the target of updating policy function of soft Q-learning, and experts may not outperform $\pi_\theta$ on $\eta(\pi)$. The reason is that experts may not be good at exploring the state-action space, but $\eta(\pi)$ encourages it. 
What expert demonstrations have in common is their relatively high returns, although the returns are not provided numerically. 
As a reduction, in this paper we define $\pi^*$ \emph{perform better} than $\pi_\theta$ by

\begin{equation*}
\mathbb{E}_{\tau\sim\pi^*}\left[\sum_{t=0}^\infty r_t\right]\geqslant \mathbb{E}_{\tau\sim\pi_\theta}\left[\sum_{t=0}^\infty r_t\right],
\end{equation*}
which demonstrates the common feature of expert demonstrations. As we prove in Appendix \ref{proof:th1}, this constraint is equivalent to the following one:

\begin{equation}\label{constraint}
\begin{split}
\mathbb{E}_{\tau\sim\pi^*}\left[\sum_{t=0}^{\infty}A_\theta(s_t^*,a_t^*)\right] + \mathbb{E}_{\tau\sim\pi_\theta} \left[ \sum_{t=0}^{\infty} \epsilon h^{\pi_\theta}_{s_t}\right]\geqslant 0.
\end{split}
\end{equation}

If the constraint (\ref{constraint}) is satisfied, by definition the Q function is accurate enough to believe that expert demonstrations perform better. Therefore we can update the Q function by forcing it to satisfy the constraint. The loss for Q function is 
\begin{equation*}
L_Q = \left[-\mathbb{E}_{\tau\sim\pi^*}\left[\sum_{t=0}^{\infty}A_\theta(s_t^*,a_t^*)\right] - \mathbb{E}_{\tau\sim\pi_\theta} \left[ \sum_{t=0}^{\infty} \epsilon h^{\pi_\theta}_{s_t}\right]\right]_+ ,
\end{equation*}
where $[x]_+=\max (0,x)$. By stopping the gradient of the policy $\pi_\theta$, we can calculate the gradient with expert demonstrations

\begin{equation}\label{gradq}
\begin{split}
g_Q &= \left.-\nabla_{\theta}L_Q\right|_{stop-\pi}\\
&= \left.\nabla_{\theta}\mathbb{E}_{\tau\sim\pi^*}\sum_{t=0}^\infty \left[Q_{\theta}(s_t^*, a_t^*) - \sum_{a'}\overline{\pi_{\theta}(a'|s_t^*)}Q_{\theta}(s_t^*, a')\right] \right|_{\alpha < 0} ,
\end{split}
\end{equation}
where 
\begin{equation*}
 \alpha=\mathbb{E}_{\tau\sim\pi^*}\sum_{t=0}^\infty r_t^* - \mathbb{E}_{\tau\sim\pi}\sum_{t=0}^{\infty}r_t .
\end{equation*}

Details of our calculation are presented in Appendix \ref{proof:g}. 
\subsection{Combining Soft Q-learning with Expert Demonstrations}
In our pretrain method, both policy gradient and Q update are based on a relatively accurate estimation of Q function for $\pi_\theta$. However, since expert demonstrations do not contain reward signals, our method has to learn together with online simulated data with reward, using soft Q-learning. So the gradient of our method is 

\begin{equation}\label{pregrad}
g^{pre} = g^{soft-Q} + \lambda(\epsilon g_\pi+g_Q),
\end{equation}
where $\lambda$ is the weighting parameter.

After pretraining with the gradient (\ref{pregrad}), we continue the training process with soft Q-learning, using only the replay buffer of online trajectories. The reason is that the experts do not perform as well as the potential of soft Q-learning and that our purpose of introducing expert demonstrations is warming up soft Q-learning and allowing for further performance improvement. 

The full version of our method is illustrated in Algorithm \ref{alg}.

\begin{algorithm}[tb]
	\caption{Pretraining Soft Q-learning with Demonstrations}
	\label{alg}
	\begin{algorithmic}
		\STATE {\bfseries Initialize:} network parameter $\theta$, pretraining step $N_p$, total training step $N$, expert buffer $D^*$, replay buffer $D$.
		\FOR{$t=0$ {\bfseries to} $N-1$}
		\STATE Simulate for a few steps with $\theta$, after each step, add $\{s_t,a_t\}$ to $D$.
		\IF{$t < N_p$}
		\STATE Take a demonstration batch from $D^*$ and a replay batch from $D$.
		\STATE Update $\theta$ with the demonstration batch and replay batch, using gradient (\ref{pregrad})
		\ELSE
		\STATE Take a replay batch from $D$.
		\STATE Update $\theta$ with the replay batch, using gradient (\ref{equv}) 
		\ENDIF
		\ENDFOR
	\end{algorithmic}
\end{algorithm}

\section{Discussion}
\label{discussion}
One of the main contributions of our method is decoupling the policy function and Q function in policy evaluation with expert demonstrations. By stopping the gradient of one factor, we can update the other. Note that soft Q-learning is a value based RL algorithm, and only the Q function is explicit, we regard the method as an actor-critic method using a double headed neural network. The input of the double headed neural network is the observation, one output is the Q function of all actions, and the other output is the policy, i.e., the probability distribution of actions. Both of the outputs share all the trainable parameters, and the policy is calculated with Equation (\ref{pitheta}). 

Since we can calculate the policy evaluation using the two heads, we can train the two heads respectively by stopping gradients. It is interesting to note that if we do not decouple the policy function and Q function, the ``policy gradient based on demonstrations'' is actually equivalent to the cross-entropy gradient for classification tasks:
\begin{equation*}
\begin{split}
-\nabla_\theta\eta(\pi_\theta) &= \nabla_\theta\mathbb{E}_{\tau\sim\pi^*}\left[\sum^{\infty}_{t=0} A^{\pi_\theta}(s_t^*,a_t^*)\right]\\
&=\nabla_{\theta}\mathbb{E}_{\tau\sim\pi^*}\left[\sum_{t=0}^\infty \left[Q_{\theta}(s_t^*, a_t^*) - \sum_{a'}\pi_{\theta}(a'|s_t^*)Q_{\theta}(s_t^*, a') + \epsilon \sum_{a'}\pi_{\theta}(a'|s_t^*)\log \pi_{\theta}(a'|s_t^*)\right]\right] \\
&=\epsilon \nabla_\theta\mathbb{E}_{\tau\sim\pi^*}\left[\sum^{\infty}_{t=0} \log \pi_{\theta}(a_t^*|s_t^*)\right], 
\end{split}
\end{equation*}
which is also the gradient of BC methods.

The derivation holds true provided that $p(a|s)$ can be always substituted with the optimal softmax policy $\frac{\exp(Q_{\theta}(s_t^*,a')/\epsilon)}{\sum_{a}\exp(Q_{\theta}(s_t^*,a)/\epsilon)}$. However, the optimal property of softmax distribution is unidirectional: a policy mimicking softmax distribution is guaranteed to be optimal in the context of any given Q estimation, but in the opposite direction, a Q value estimation whose softmax $\pi_{\theta}(a|s)$ mimicking an expert policy is not guaranteed to be a more accurate estimation in the environment, which limits the agent's ability of generalizing to trajectories out of expert demonstrations. 

Above is the reason why we decouple the policy function and Q function in policy evaluation with demonstrations. Experiments also show that our method outperforms Behavior Cloning methods on the pretraining task.

We train $Q_\theta$ with the constraint that $\pi^*$ performs better than $\pi_\theta$. The method is actually a modified version of IRL. Note that 
\begin{equation*}
\pi_\theta(a|s)=\argmax_\pi \left({\mathbb{E}_{a'\sim \pi}\left[Q^{\pi_\theta}(s,a')\right]}+\epsilon h^\pi_{s}\right),
\end{equation*}
then we can rewrite the gradient (\ref{gradq}) for Q as 

\begin{equation*}
g_Q =\nabla_\theta \mathbb{E}_{\tau\sim\pi^*}\Bigg[\sum_{t=0}^{\infty} \Big[ Q_\theta(s_t^*,a_t^*)-
\max_\pi \left({\mathbb{E}_{a'\sim \pi}\left[Q_\theta(s^*_t,a')\right]}+\epsilon h^\pi_{s^*_t}\right)\Big]\Bigg]\Bigg|_{\alpha< 0},
\end{equation*}
which is a soft Q-learning version of IRL, where $\alpha=\mathbb{E}_{\tau\sim\pi^*}\left[\sum_{t=0}^{\infty}A_\theta(s_t^*,a_t^*)\right] + \mathbb{E}_{\tau\sim\pi_\theta} \left[ \sum_{t=0}^{\infty} \epsilon h^{\pi_\theta}_{s_t}\right]$. Here the back-propagation of $\pi_\theta$ is stopped, and we ignore the terms that have $\pi_\theta$ in the gradient. And we set the condition $\alpha< 0$ for this IRL process, because we only want the Q function estimator to satisfy constraint (\ref{constraint}). This process is quite similar to DQfD \cite{hester2017deep}, except that this gradient have a constraint and entropy term, and the margin used in DQfD is abandoned.

The constraint $\alpha<0$ is non-heuristic for value based RL algorithms. It is based on our understanding that soft Q-learning is equivalent to a policy gradient method and that the implicit policy for this method is $\pi_\theta(a|s)$ defined in Equation (\ref{pitheta}). The constraint is obtained from constraint (\ref{constraint}) using the chain rule and is applied on the expectation of value functions, instead of the value function on a single step, which is different from the ``Q-filter'' that Nair et al. proposed \cite{nair2018overcoming}.

\section{Experiments}
\subsection{Experimental Setup}
Our method is tested on various environments from Atari 2600 video games. These environments provide images as observations and have discrete action spaces. Example screenshots are demonstrated in Figure \ref{fig:environments}. In this paper, we use down-sampled $84\times84$ gray scale images from video screens as states for the agent, and clipped reward signal from the environment for training. In order to demonstrate the performance of our method clearly, we use no-clipped reward when reporting experimental results. 

\begin{figure}[tb]
	\centering
	\includegraphics[width=0.24\textwidth]{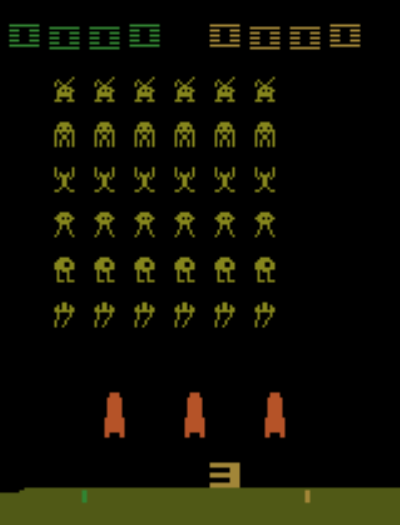}
	\includegraphics[width=0.24\textwidth]{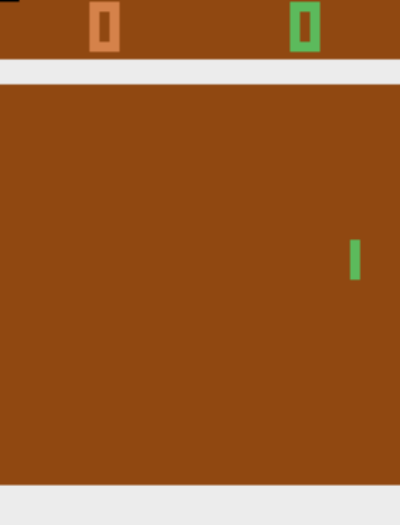}
	\includegraphics[width=0.24\textwidth]{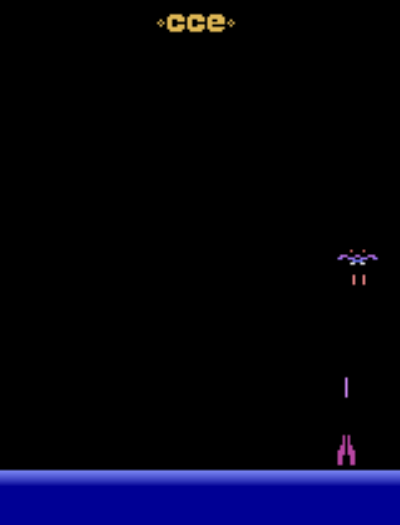}
	\includegraphics[width=0.24\textwidth]{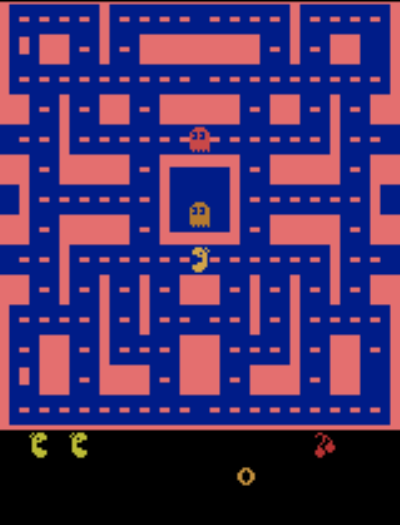}
	\caption{Example screenshots of Atari simulation environments that we experiment on. From left to right are: SpaceInvaders, Pong, DemonAttack, MsPacman.}
	\label{fig:environments}
\end{figure}

The datasets of expert demonstrations used in the experiments contain trajectories of $\{s_t, a_t\}$. Each dataset for one environment contains around 40000 time steps. The datasets are generated from an agent trained with PPO \cite{schulman2017proximal}, though our method is not limited to using demonstrations rolled out by specific algorithms. All the methods in experiments pretrain with the same demonstration datasets for an equitable comparison.

The agent of our proposed method uses a 5-layer neural network, which contains 3 convolutional layers and 2 fully connected layers. The inputs of the network are batches of images, and the outputs are Q values of each action. 

In this paper, we compare our method with BC and DQfD among all the tasks. Due to different learning speed of different games, the pretraining steps $N_p$ vary among tasks. Other hyper parameters including replay buffer size, expert buffer size and batch size are identical in all tasks. Our parameters setting is listed in Appendix \ref{appendix:param}.

\subsection{Training Process}
The agent is trained with the process shown in Algorithm \ref{alg}. During initialization, the replay buffer is filled with roll outs from a random policy. Afterwards, the agent enters pretrain phase and learns from both expert demonstrations and online simulated data. Every $\textrm{learning\_frequency}$ steps, the agent samples the same number of data from expert demonstrations and replay buffer, and updates Q value using gradient \ref{pregrad}.

After pretrain, the agent continues training via a normal soft Q-learning, sampling only from replay buffer. In all of the four experiment environments, we train the agent for a total of $5\times10^6$ time steps.

Online trajectories are rolled out according to $\pi_{\theta}(a|s)$ (Equation \ref{pitheta}) and stored in replay buffer along with the training process as a common practice to mitigate distributional shift by online data augmentation.  

\subsection{Experimental Results}
The training curves demonstrated in Figure \ref{fig:experiment} show the effectiveness of our pretrain method. 

\paragraph{Effect of pretrain.} 
Even if the performance of demonstrations is imperfect (see dashed horizontal line in the first three plots), our method still manages to exceed demonstration level via online updating after pretrain phase.

\paragraph{Decoupling policy and value function.}The effect of decoupling soft Q-learning into policy and value can be seen by comparing BC method with ours, since we have proved in section \ref{discussion} that BC is equivalent to directly training with our loss function in a normal soft Q-learning architecture. The results show that decoupled version is more adaptive to further online updating after pretrain. It is interesting to observe that BC goes through a sharp drop and restart when pretrain phase finishes, while our method enjoys a smoother transition from pretrain to normal soft Q-learning. This can be explained as BC aims to mimic only the policy of demonstrations rather than carefully adjust policy and value estimation as we have done.

\paragraph{Non-heuristic loss design.} DQfD learns from demonstrations by forcing a margin between the Q value of the expert action and other actions, and our method gets rid of this heuristic margin design. Although DQfD learns super fast during pretrain, our method again is more applicable to learning after pretrain because our loss functions are theoretically derived and are not intuitively manipulating value functions.

\begin{figure}
	\centering
	\includegraphics[width=1\textwidth]{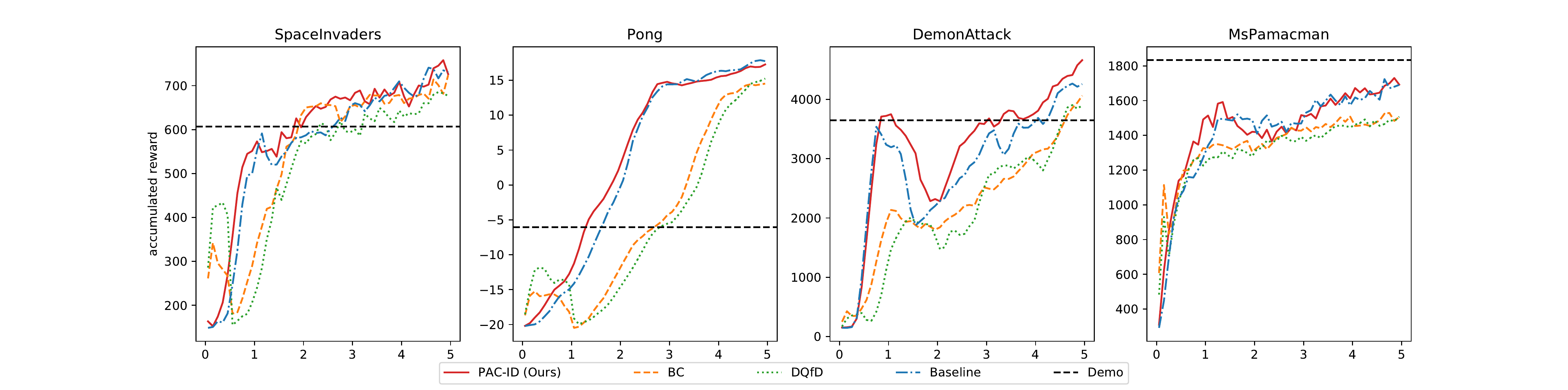}
	\caption{experiments on atari 2600}
	\label{fig:experiment}
\end{figure}

\section{Conclusion}
In this paper, we propose a method that pretrains soft Q-learning with imperfect demonstrations. The method is based on the equivalence between soft Q-learning and policy gradient methods. We prove that for the undiscounted policy evaluation with entropy, we can present the policy evaluation with reward-missing expert demonstrations. Then with this result, we introduce expert demonstrations to soft Q-learning. The method decouples the Q function and the policy function in the policy evaluation and updates both functions respectively by stopping the gradient back-propagation one at a time. In the experiments, we demonstrate that on various tasks, our method outperforms BC and DQfD concerning pretraining soft Q-learning.

In this paper, the proposed method is based on the entropy regularizer and is not applicable to other value based RL algorithms. We hope to solve this problem in our future work. Also, as our method is a single-step gradient based method, we can extend our work with methods such as trust region methods and multi-step learning. We also leave these extensions in our future work.

\section*{Acknowledgments}
This work was supported by the National Key Basic Research
Program of China (No. 2016YFB0100900) and National Natural Science Foundation of China (No. 61773231).

\newpage
\section*{Appendix A}
\setcounter{section}{0}
\setcounter{theorem}{0}
\renewcommand\thesection{A.\arabic{section}} 

\section{Learning from Imperfect Demonstrations without Reward}\label{proof:th1}

\begin{theorem}
	In the settings of undiscounted policy evaluation with entropy, we have
	\begin{equation}
	\eta(\pi^*)-\eta(\pi) = \mathbb{E}_{\tau\sim\pi^*}\left[\sum^{\infty}_{t=0}A^\pi(s_t^*,a_t^*) + \epsilon h^{\pi^*}_{s_t^*}\right].
	\end{equation}
	\begin{equation}
	\mathbb{E}_{\tau\sim\pi_\theta}\left[\sum_{t=0}^\infty r_t\right]
	=\mathbb{E}_{\tau\sim\pi^*}\left[\sum_{t=0}^\infty r_t\right] - \left[\mathbb{E}_{\tau\sim\pi^*}\left[\sum_{t=0}^{\infty}A^{\pi}(s_t^*,a_t^*)\right] + \mathbb{E}_{\tau\sim\pi_\theta} \left[ \sum_{t=0}^{\infty} \epsilon h^{\pi}_{s_t}\right]\right].
	\end{equation}
	
\end{theorem}

\begin{proof}
	
	From definition of advantage function, we have
	
	\begin{equation}\label{aaa}
	A^\pi(s_t,a_t) = \mathbb{E}_{r_t,s_{t+1}|s_t,a_t}\left[r_t+V^\pi(s_{t+1})-V^\pi(s_t)\right].
	\end{equation}
	
	hence 
	\begin{equation*}
	\begin{split}
	&\mathbb{E}_{\tau\sim\pi^*}\left[\sum_{t=0}^{\infty}A^\pi(s_t^*,a_t^*)\right]\\
	=&\mathbb{E}_{\tau\sim\pi^*}\left[\sum_{t=0}^{\infty}r_t+V^\pi(s_{t+1}^*)-V^\pi(s_t^*)\right]\\
	=&\mathbb{E}_{\tau\sim\pi^*}\left[-V^\pi(s_0^*)+\sum_{t=0}^{\infty}r_t\right]\\
	=&-\mathbb{E}_{s_0} \left[ V^\pi(s_0^*) \right] + \mathbb{E}_{\tau\sim\pi^*} \left[ \sum_{t=0}^{\infty} r_t + \epsilon h^{\pi^*}_{s_t^*}\right] - \mathbb{E}_{\tau\sim\pi^*} \left[ \sum_{t=0}^{\infty} \epsilon h^{\pi^*}_{s_t^*}\right]\\
	=&-\eta(\pi) + \eta(\pi^*) - \mathbb{E}_{\tau\sim\pi^*} \left[ \sum_{t=0}^{\infty} \epsilon h^{\pi^*}_{s_t^*}\right].
	\end{split}
	\end{equation*}
	
	Rearrange the result, the theorem holds.
\end{proof}

Similar to \cite{zhang2018pretraining}, we constrain expert policies $\pi^*$ \emph{perform better} than $\pi_\theta(\cdot|\cdot)$ defined in equation (\ref{pitheta}). In this paper, we define $\pi^*$ \emph{perform better} than $\pi_\theta$ by

\begin{equation}
\mathbb{E}_{\tau\sim\pi^*}\left[\sum_{t=0}^\infty r_t\right]\geqslant \mathbb{E}_{\tau\sim\pi_\theta}\left[\sum_{t=0}^\infty r_t\right],
\end{equation}

and by definition and Theorem \ref{theo}, we have 

\begin{equation*}
\begin{split}
&\mathbb{E}_{\tau\sim\pi^*}\left[\sum_{t=0}^\infty r_t\right] -  \mathbb{E}_{\tau\sim\pi_\theta}\left[\sum_{t=0}^\infty r_t\right]\\
=&-\eta(\pi_\theta) + \eta(\pi^*) - \mathbb{E}_{\tau\sim\pi^*} \left[ \sum_{t=0}^{\infty} \epsilon h^{\pi^*}_{s_t^*}\right] + \mathbb{E}_{\tau\sim\pi_\theta} \left[ \sum_{t=0}^{\infty} \epsilon h^{\pi_\theta}_{s_t}\right] \\
=&\mathbb{E}_{\tau\sim\pi^*}\left[\sum_{t=0}^{\infty}A_\theta(s_t^*,a_t^*)\right] + \mathbb{E}_{\tau\sim\pi_\theta} \left[ \sum_{t=0}^{\infty} \epsilon h^{\pi_\theta}_{s_t}\right]\geqslant 0.
\end{split}
\end{equation*}

Hence the loss function of $\pi_\theta$ is 

\begin{equation}
L_\pi = \mathbb{E}_{\tau\sim\pi^*}\left[\sum_{t=0}^{\infty}A_\theta(s_t^*,a_t^*)\right] + \mathbb{E}_{\tau\sim\pi_\theta} \left[ \sum_{t=0}^{\infty} \epsilon h^{\pi_\theta}_{s_t}\right],
\end{equation}

and the loss function for Q function is

\begin{equation}
L_Q = \left[-\mathbb{E}_{\tau\sim\pi^*}\left[\sum_{t=0}^{\infty}A_\theta(s_t^*,a_t^*)\right] - \mathbb{E}_{\tau\sim\pi_\theta} \left[ \sum_{t=0}^{\infty} \epsilon h^{\pi_\theta}_{s_t}\right]\right]_+,
\end{equation}

\section{Gradients of Decoupled Policy and Q Functions}\label{proof:g}
Inspired by Zhang and Ma \cite{zhang2018pretraining}, since soft Q-learning is equivalent to actor-critic algorithms, by stopping the gradient of the Q function, we calculate policy gradient corresponding to $L_\pi$ with expert demonstrations.





\begin{equation}
\begin{split}
g_\pi &=-\left.\nabla_{\theta}L_{\pi}\right|_{stop-Q} \\
&= -\left.\nabla_\theta\left[ \mathbb{E}_{\tau\sim\pi^*} \left[ \sum^{\infty}_{t=0} A_\theta(s_t^*,a_t^*)\right] + \mathbb{E}_{\tau\sim\pi}\left[\sum_{t=0}^\infty\epsilon h_{s_t}^{\pi_{\theta}}\right]\right]\right|_{stop-Q}\\
&= -\nabla_\theta \mathbb{E}_{\tau\sim\pi^*} \left[ \sum^{\infty}_{t=0} \overline{Q_\theta(s_t,a_t)}-\sum_{a'}\pi_\theta(a'|s_t)\overline{Q_\theta(s_t,a')}-\epsilon h^{\pi_\theta}_{s_t^*} \right] + \mathbb{E}_{\tau\sim\pi}\left[\sum_{t=0}^\infty \epsilon h_{s_t}^{\pi_{\theta}}\right]\\
&=\mathbb{E}_{\tau^*\sim\pi^*,\tau\sim\pi, a'\sim \pi_{\theta}(a'|s_t^*)}\sum_{t=0}^\infty \left[\overline{Q_{\theta}(s_t^*,a')}\nabla_{\theta}\log \pi_{\theta}(a'|s_t^*)+\epsilon \nabla_{\theta}\left( h_{s_t^*}^{\pi_{\theta}} -h_{s_t}^{\pi_{\theta}}\right)\right],
\end{split}
\end{equation}

where $\nabla_\theta\overline{Q_\theta(s,a)}=0$.

Based on the assumption that expert demonstration performs better than the agent in pretraining period, we can update Q function to match this assumption when the constraint (\ref{constraint}) is violated. 
By stopping the gradient of the policy $\pi_\theta$, 
we can calculate the gradient of Q function
\begin{equation}\label{eqn:gq1}
\begin{split}
g_Q &= \left.-\nabla_{\theta}L_Q\right|_{stop-\pi} \\
&=\nabla_{\theta}\left.\left[\mathbb{E}_{\tau\sim\pi^*}\left[\sum_{t=0}^{\infty}A_\theta(s_t^*,a_t^*)\right] + \mathbb{E}_{\tau\sim\pi_\theta} \left[ \sum_{t=0}^{\infty} \epsilon h^{\overline{\pi_\theta}}_{s_t}\right]\right]\right|_{\alpha < 0} \\
&= \left.\nabla_{\theta}\mathbb{E}_{\tau\sim\pi^*}\sum_{t=0}^\infty \left[Q_{\theta}(s_t^*, a_t^*) - \sum_{a'}\overline{\pi_{\theta}(a'|s_t^*)}Q_{\theta}(s_t^*, a')\right] \right|_{\alpha < 0} 
\end{split}
\end{equation}
\begin{equation}
\begin{split}
\alpha &= \mathbb{E}_{\tau\sim\pi^*}\sum_{t=0}^\infty r_t^* - \mathbb{E}_{\tau\sim\pi}\sum_{t=0}^{\infty}r_t \\
\end{split}
\end{equation}


where $\nabla_\theta \overline{\pi_\theta(a|s)} = 0$, 
and we set $\nabla_x [x]_+=0$ when $x=0$.

\section{Experiment Hyperparameters}\label{appendix:param}
\begin{table}[!h]
	\centering
	\begin{tabular}{l l}
		\hline
		Environment & $N_p$ \\
		\hline
		SpaceInvaders & $5 \times 10^5 $\\
		Pong & $1 \times 10^6$ \\
		DemonAttack & $5 \times 10^5$ \\
		MsPacman & $2 \times 10^5$ \\
		\hline		
	\end{tabular}
	\caption{Pretrain length for different environments.}
	\label{tab:pretrain_length}
\end{table}
\begin{table}[!h]
	\begin{tabular}{l l l}
		\hline
		Hyperparameter & Value & Description \\
		\hline
		minibatch size & 32 & Number of training cases the optimizer computed over. \\
		replay buffer size & 1000000 & Number of history steps the agent keeps memory of. \\
		initial learning rate & 0.0001 & Learning rate used by Adam in time steps [0, $1\times 10^6$]. \\
		final learning rate & 0.00005 & Learning rate decreases linearly to this value. \\
		replay start size & 50000 & Random steps rolled out before learning starts. \\
		learning frequency & 4 & Frequency of gradient updating. \\
		target update frequency & 10000 & Frequency of target network updating. \\	
		$\epsilon$ & 0.1 & Weighting parameter of entropy regularizer. \\
		$\lambda$ & 1.0 & Weighting parameter of pretrain loss. \\
		maximum timesteps & $5\times 10^6$ & Total training steps. \\
		\hline
	\end{tabular}
	\caption{List of hyperparameters.}
	\label{tab:hyperparameters}
\end{table}

For each of the environments in our experiment, we set different length of pretrain steps as presented in Table \ref{tab:pretrain_length}.

Other hyperparameters are identical among environments, and we list them in Table \ref{tab:hyperparameters}.
\vskip 0.2in
\bibliographystyle{theapa}
\bibliography{sample}

\begin{thebibliography}{}

\bibitem[\protect\BCAY{Brys, Harutyunyan, Suay, Chernova, Taylor,\ \BBA\
  Now{\'e}}{Brys et~al.}{2015}]{brys2015reinforcement}
Brys, T., Harutyunyan, A., Suay, H.~B., Chernova, S., Taylor, M.~E., \BBA\
  Now{\'e}, A. \BBOP2015\BBCP.
\newblock \BBOQ Reinforcement learning from demonstration through
  shaping.\BBCQ\
\newblock In {\Bem IJCAI}, \BPGS\ 3352--3358.

\bibitem[\protect\BCAY{Haarnoja, Tang, Abbeel,\ \BBA\ Levine}{Haarnoja
  et~al.}{2017}]{haarnoja17reinforcement}
Haarnoja, T., Tang, H., Abbeel, P., \BBA\ Levine, S. \BBOP2017\BBCP.
\newblock \BBOQ Reinforcement learning with deep energy-based policies\BBCQ\
\newblock In {\Bem Proceedings of the 34th International Conference on Machine
  Learning}, \lowercase{\BVOL}~70 of {\Bem Proceedings of Machine Learning
  Research}, \BPGS\ 1352--1361.

\bibitem[\protect\BCAY{Hessel, Modayil, Van~Hasselt, Schaul, Ostrovski, Dabney,
  Horgan, Piot, Azar,\ \BBA\ Silver}{Hessel et~al.}{2017}]{hessel2017rainbow}
Hessel, M., Modayil, J., Van~Hasselt, H., Schaul, T., Ostrovski, G., Dabney,
  W., Horgan, D., Piot, B., Azar, M., \BBA\ Silver, D. \BBOP2017\BBCP.
\newblock \BBOQ Rainbow: Combining improvements in deep reinforcement
  learning\BBCQ\
\newblock {\Bem arXiv preprint arXiv:1710.02298}.

\bibitem[\protect\BCAY{Hester, Vecerik, Pietquin, Lanctot, Schaul, Piot,
  Horgan, Quan, Sendonaris, Dulac-Arnold, et~al.}{Hester
  et~al.}{2017}]{hester2017deep}
Hester, T., Vecerik, M., Pietquin, O., Lanctot, M., Schaul, T., Piot, B.,
  Horgan, D., Quan, J., Sendonaris, A., Dulac-Arnold, G., et~al.
  \BBOP2017\BBCP.
\newblock \BBOQ Deep q-learning from demonstrations\BBCQ\
\newblock {\Bem arXiv preprint arXiv:1704.03732}.

\bibitem[\protect\BCAY{Kang, Jie,\ \BBA\ Feng}{Kang
  et~al.}{2018}]{kang2018policy}
Kang, B., Jie, Z., \BBA\ Feng, J. \BBOP2018\BBCP.
\newblock \BBOQ Policy optimization with demonstrations\BBCQ\
\newblock In {\Bem International Conference on Machine Learning}, \BPGS\
  2474--2483.

\bibitem[\protect\BCAY{Lakshminarayanan, Ozair,\ \BBA\ Bengio}{Lakshminarayanan
  et~al.}{2016}]{lakshminarayanan2016reinforcement}
Lakshminarayanan, A.~S., Ozair, S., \BBA\ Bengio, Y. \BBOP2016\BBCP.
\newblock \BBOQ Reinforcement learning with few expert demonstrations\BBCQ\
\newblock In {\Bem NIPS Workshop on Deep Learning for Action and Interaction},
  \lowercase{\BVOL}\ 2016.

\bibitem[\protect\BCAY{Lillicrap, Hunt, Pritzel, Heess, Erez, Tassa, Silver,\
  \BBA\ Wierstra}{Lillicrap et~al.}{2015}]{lillicrap2015continuous}
Lillicrap, T.~P., Hunt, J.~J., Pritzel, A., Heess, N., Erez, T., Tassa, Y.,
  Silver, D., \BBA\ Wierstra, D. \BBOP2015\BBCP.
\newblock \BBOQ Continuous control with deep reinforcement learning\BBCQ\
\newblock {\Bem arXiv preprint arXiv:1509.02971}.

\bibitem[\protect\BCAY{Mnih, Badia, Mirza, Graves, Lillicrap, Harley, Silver,\
  \BBA\ Kavukcuoglu}{Mnih et~al.}{2016}]{mnih2016asynchronous}
Mnih, V., Badia, A.~P., Mirza, M., Graves, A., Lillicrap, T., Harley, T.,
  Silver, D., \BBA\ Kavukcuoglu, K. \BBOP2016\BBCP.
\newblock \BBOQ Asynchronous methods for deep reinforcement learning\BBCQ\
\newblock In {\Bem International conference on machine learning}, \BPGS\
  1928--1937.

\bibitem[\protect\BCAY{Mnih, Kavukcuoglu, Silver, Rusu, Veness, Bellemare,
  Graves, Riedmiller, Fidjeland, Ostrovski, et~al.}{Mnih
  et~al.}{2015}]{mnih2015human}
Mnih, V., Kavukcuoglu, K., Silver, D., Rusu, A.~A., Veness, J., Bellemare,
  M.~G., Graves, A., Riedmiller, M., Fidjeland, A.~K., Ostrovski, G., et~al.
  \BBOP2015\BBCP.
\newblock \BBOQ Human-level control through deep reinforcement learning\BBCQ\
\newblock {\Bem Nature}, {\Bem 518\/}(7540), 529.

\bibitem[\protect\BCAY{Nachum, Norouzi, Xu,\ \BBA\ Schuurmans}{Nachum
  et~al.}{2017}]{nachum2017bridging}
Nachum, O., Norouzi, M., Xu, K., \BBA\ Schuurmans, D. \BBOP2017\BBCP.
\newblock \BBOQ Bridging the gap between value and policy based reinforcement
  learning\BBCQ\
\newblock In {\Bem Advances in Neural Information Processing Systems}, \BPGS\
  2775--2785.

\bibitem[\protect\BCAY{Nair, McGrew, Andrychowicz, Zaremba,\ \BBA\ Abbeel}{Nair
  et~al.}{2018}]{nair2018overcoming}
Nair, A., McGrew, B., Andrychowicz, M., Zaremba, W., \BBA\ Abbeel, P.
  \BBOP2018\BBCP.
\newblock \BBOQ Overcoming exploration in reinforcement learning with
  demonstrations\BBCQ\
\newblock In {\Bem 2018 IEEE International Conference on Robotics and
  Automation (ICRA)}, \BPGS\ 6292--6299. IEEE.

\bibitem[\protect\BCAY{Piot, Geist,\ \BBA\ Pietquin}{Piot
  et~al.}{2014}]{piot2014boosted}
Piot, B., Geist, M., \BBA\ Pietquin, O. \BBOP2014\BBCP.
\newblock \BBOQ Boosted bellman residual minimization handling expert
  demonstrations\BBCQ\
\newblock In {\Bem Joint European Conference on Machine Learning and Knowledge
  Discovery in Databases}, \BPGS\ 549--564. Springer.

\bibitem[\protect\BCAY{Rajeswaran, Kumar, Gupta, Vezzani, Schulman, Todorov,\
  \BBA\ Levine}{Rajeswaran et~al.}{2017}]{rajeswaran2017learning}
Rajeswaran, A., Kumar, V., Gupta, A., Vezzani, G., Schulman, J., Todorov, E.,
  \BBA\ Levine, S. \BBOP2017\BBCP.
\newblock \BBOQ Learning complex dexterous manipulation with deep reinforcement
  learning and demonstrations\BBCQ\
\newblock {\Bem arXiv preprint arXiv:1709.10087}.

\bibitem[\protect\BCAY{Schulman, Chen,\ \BBA\ Abbeel}{Schulman
  et~al.}{2017}]{schulman2017equivalence}
Schulman, J., Chen, X., \BBA\ Abbeel, P. \BBOP2017\BBCP.
\newblock \BBOQ Equivalence between policy gradients and soft q-learning\BBCQ\
\newblock {\Bem arXiv preprint arXiv:1704.06440}.

\bibitem[\protect\BCAY{Schulman, Moritz, Levine, Jordan,\ \BBA\
  Abbeel}{Schulman et~al.}{2016}]{schulman2015high}
Schulman, J., Moritz, P., Levine, S., Jordan, M., \BBA\ Abbeel, P.
  \BBOP2016\BBCP.
\newblock \BBOQ High-dimensional continuous control using generalized advantage
  estimation\BBCQ\
\newblock {\Bem International Conference on Learning Representations}.

\bibitem[\protect\BCAY{Schulman, Wolski, Dhariwal, Radford,\ \BBA\
  Klimov}{Schulman et~al.}{2017}]{schulman2017proximal}
Schulman, J., Wolski, F., Dhariwal, P., Radford, A., \BBA\ Klimov, O.
  \BBOP2017\BBCP.
\newblock \BBOQ Proximal policy optimization algorithms\BBCQ\
\newblock {\Bem arXiv preprint arXiv:1707.06347}.

\bibitem[\protect\BCAY{Silver, Huang, Maddison, Guez, Sifre, Van Den~Driessche,
  Schrittwieser, Antonoglou, Panneershelvam, Lanctot, et~al.}{Silver
  et~al.}{2016}]{silver2016mastering}
Silver, D., Huang, A., Maddison, C.~J., Guez, A., Sifre, L., Van Den~Driessche,
  G., Schrittwieser, J., Antonoglou, I., Panneershelvam, V., Lanctot, M.,
  et~al. \BBOP2016\BBCP.
\newblock \BBOQ Mastering the game of go with deep neural networks and tree
  search\BBCQ\
\newblock {\Bem nature}, {\Bem 529\/}(7587), 484.

\bibitem[\protect\BCAY{Silver, Hubert, Schrittwieser, Antonoglou, Lai, Guez,
  Lanctot, Sifre, Kumaran, Graepel, et~al.}{Silver
  et~al.}{2018}]{silver2018general}
Silver, D., Hubert, T., Schrittwieser, J., Antonoglou, I., Lai, M., Guez, A.,
  Lanctot, M., Sifre, L., Kumaran, D., Graepel, T., et~al. \BBOP2018\BBCP.
\newblock \BBOQ A general reinforcement learning algorithm that masters chess,
  shogi, and go through self-play\BBCQ\
\newblock {\Bem Science}, {\Bem 362\/}(6419), 1140--1144.

\bibitem[\protect\BCAY{Sutton\ \BBA\ Barto}{Sutton\ \BBA\
  Barto}{2018}]{sutton2018reinforcement}
Sutton, R.~S.\BBACOMMA\  \BBA\ Barto, A.~G. \BBOP2018\BBCP.
\newblock {\Bem Reinforcement learning: An introduction}.
\newblock MIT press.

\bibitem[\protect\BCAY{Sutton, McAllester, Singh,\ \BBA\ Mansour}{Sutton
  et~al.}{2000}]{sutton2000policy}
Sutton, R.~S., McAllester, D.~A., Singh, S.~P., \BBA\ Mansour, Y.
  \BBOP2000\BBCP.
\newblock \BBOQ Policy gradient methods for reinforcement learning with
  function approximation\BBCQ\
\newblock In {\Bem Advances in neural information processing systems}, \BPGS\
  1057--1063.

\bibitem[\protect\BCAY{Vecer{\'\i}k, Hester, Scholz, Wang, Pietquin, Piot,
  Heess, Roth{\"o}rl, Lampe,\ \BBA\ Riedmiller}{Vecer{\'\i}k
  et~al.}{2017}]{vecerik2017leveraging}
Vecer{\'\i}k, M., Hester, T., Scholz, J., Wang, F., Pietquin, O., Piot, B.,
  Heess, N., Roth{\"o}rl, T., Lampe, T., \BBA\ Riedmiller, M.~A.
  \BBOP2017\BBCP.
\newblock \BBOQ Leveraging demonstrations for deep reinforcement learning on
  robotics problems with sparse rewards\BBCQ\
\newblock {\Bem CoRR, abs/1707.08817}.

\bibitem[\protect\BCAY{Zhang\ \BBA\ Ma}{Zhang\ \BBA\
  Ma}{2018}]{zhang2018pretraining}
Zhang, X.\BBACOMMA\  \BBA\ Ma, H. \BBOP2018\BBCP.
\newblock \BBOQ Pretraining deep actor-critic reinforcement learning algorithms
  with expert demonstrations\BBCQ\
\newblock {\Bem arXiv preprint arXiv:1801.10459}.

\end{thebibliography}

\end{document}